\documentclass[letterpaper]{article}
\usepackage{aaai2027}
\usepackage[hyphens]{url}
\usepackage{graphicx}
\usepackage{natbib}
\usepackage{caption}
\usepackage{algorithm}
\usepackage{algorithmic}
\usepackage{booktabs}
\usepackage{multirow}
\usepackage{colortbl}
\usepackage{amsmath}
\usepackage{amssymb}
\usepackage{amsthm}
\usepackage{xspace}
\usepackage{tabularx}

\newcommand{\method}{\textbf{GraphBU}\xspace}
\definecolor{ourhighlight}{RGB}{255,246,204}

\newtheorem{proposition}{Proposition}

\title{\method: MILP Instance Generation with Graph-Native Block Units}
\author{
Xiaolei Guo\textsuperscript{\rm 1},
Chenyu Zhou\textsuperscript{\rm 1},
Jianghao Lin\textsuperscript{\rm 1}\corresponding,
Dongdong Ge\textsuperscript{\rm 1}
}

\affiliations{
\textsuperscript{\rm 1}Shanghai Jiao Tong University\\
\{lionelgxl, chenyuzhou, linjianghao, ddge\}@sjtu.edu.cn\\

}

\begin{document}

\maketitle

\begin{abstract}
Mixed-integer linear programming (MILP) instances used for solver development are hard to obtain when models come from private or application-specific pipelines. A generator must keep the structure that solvers and learned policies rely on. Existing general generators usually choose their generation unit from a formulation template, summary statistics, local graph edits, or blocks found after recombination. These units do not explicitly record how a local part of the MILP is coupled to the rest of the instance. We propose \method, a graph-native generator whose basic unit is a local subproblem plus its interface. The method promotes coupling nodes into master constraints or boundary variables and uses the resulting block units for compatibility-checked replacement. The analysis focuses on the properties needed by this construction: promotion separates interfaces, replacement can preserve feasibility under an interface-slack condition, and the graph construction is invariant to row-column permutations. On MILP instances generation, this unit keeps graph statistics close to the source family, preserves feasibility on most datasets, and improves downstream Predict-and-Search training. Genrated by \method, The average graph-statistical similarity was approximately 0.934, the average feasibility was approximately 96.7\%, and the average increase in the main index of downstream PS was approximately 8.0\%.
\end{abstract}

\section{Introduction}

Mixed-integer linear programming (MILP) is a standard way to model decisions that mix discrete choices with linear constraints, and it appears in applications such as scheduling, planning, logistics, and chip design. In deployed optimization systems, even small changes in solve time can affect throughput, resource utilization, or service quality. Solver development therefore depends on representative instance data: classical solvers use it for parameter tuning and stress testing, while learning-based solvers require it for training\cite{bengio2021mlco,khalil2016learning,nair2020neuralmip,han2023predictsearch}. Benchmarks also need various instances to expose failures that a small test set would miss. The difficulty is that real MILP instances are often expensive to collect, tied to private business data, or generated inside closed modeling pipelines. This motivates the task of MILP instance generation, which aims to produce additional instances that behave like a target family without requiring access to its original modeling pipeline.

\begin{figure}[!t]
\centering
\includegraphics[width=0.82\columnwidth,trim=85pt 86pt 85pt 6pt,clip]{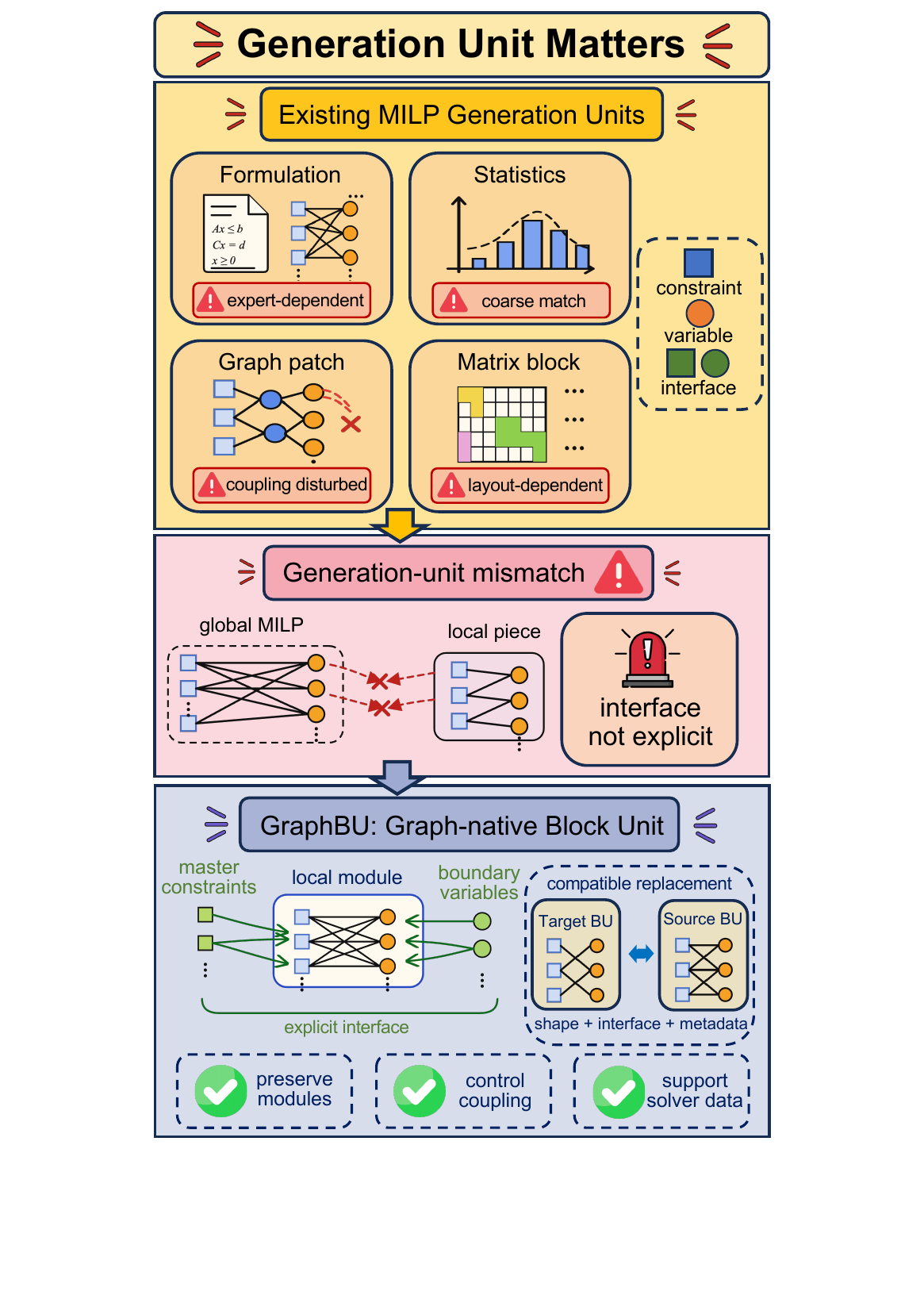}
\caption{Generation-unit mismatch in MILP instance generation. \method uses graph-native block units with explicit interfaces.}
\label{fig:intro_generation_unit}
\end{figure}

A generated MILP is useful only if it remains faithful to the target instance family. Feasibility alone is only a minimal requirement. The generated instances should preserve the scale, sparsity, coefficient patterns, feasibility behavior, and solving difficulty of the source family. This requirement is especially strict for learning-based solver modules, which learn policies from the graph structure of training instances\cite{gasse2019exact,prouvost2020ecole}. If generation changes how constraints and variables are connected, a policy trained on the generated data may not transfer to the original family. The choice of generation unit therefore becomes a design issue: the generator must decide which part of a MILP can be reused without breaking the structure that solvers rely on.

Existing generators expose the same issue from different angles. When the formulation is known, new instances can be sampled from the model template itself\cite{bowly2019stress}. This is often reliable, but it assumes access to the problem type and its mathematical description. Once that assumption is removed, many generators fall back to coarser signals, such as density, degree, or coefficient distributions\cite{bowly2019stress}. These signals are easy to measure, but they do not say which constraints and variables should be treated as one reusable part. Graph-based generation moves closer to the representation used by learned solvers; a local graph edit, however, can still cut through the coupling between a subproblem and the rest of the model\cite{geng2023g2milp}. Matrix-block methods reuse repeated submatrices, but the resulting blocks depend on row-column ordering and do not explicitly encode how a block reconnects to the remaining rows and columns\cite{liu2024milpstudio}. What is missing is the unit itself: a reusable part of the MILP should contain a graph-local subproblem and the interface through which it attaches to the full instance.

We build \method around this unit. \method represents a MILP as a constraint-variable bipartite graph, finds nodes that behave like couplings, and removes them before decomposing the residual graph. If a residual edge still crosses two local blocks, one endpoint is promoted into the interface. Each block unit then contains a local subproblem and the adjacent interface slices needed to reconnect it. Generation replaces a target unit with a source unit only when their shapes, interface dimensions, variable domains, and row metadata match.

In experiments on four MILP families, \method improves the two properties that matter for generated solver data. First, the generated instances remain close to the source families in graph-statistical similarity while preserving feasibility and nontrivial solving behavior. Second, the same generated data improves downstream Predict-and-Search training on original test instances, suggesting that the preserved block-unit structure is useful beyond matching summary statistics.

Our contributions are:
\begin{itemize}
\item We introduce \method, the first graph-native block unit for MILP instance generation. Each unit pairs a local constraint-variable module with an explicit coupling interface, and \method operationalizes this unit through interface detection, block-unit extraction, reusable library construction, and compatibility-checked replacement.
\item We provide theoretical guarantees for the construction, including interface separation, a sufficient feasibility condition under interface slack, and invariance to row-column permutations.
\item We show on four MILP families that \method better preserves graph-statistical similarity and feasibility than representative generators, and that the generated data improves downstream Predict-and-Search on held-out original instances.
\end{itemize}

\section{Related Work}

\textbf{MILP data for solver development.}
Large MILP benchmarks are hard to curate. MIPLIB 2017, for example, uses a data-driven selection process to avoid a benchmark dominated by near-duplicate or trivial instances \cite{miplib2017}. ML4CO makes the same point from the learning side: learned solver components are evaluated on distributions of related instances, not on isolated examples \cite{gasse2022ml4co}. For generated data to play the same role, it must preserve the properties that these benchmarks were selected to test.

\textbf{Learning-enhanced MILP solvers.}
Learning-based solvers often represent a MILP as a constraint-variable bipartite graph \cite{gasse2019exact,prouvost2020ecole}. This representation has been used for branching, neural MIP heuristics, and Predict-and-Search training \cite{khalil2016learning,nair2020neuralmip,han2023predictsearch}. The data requirement is strict: generated instances should expose the same graph patterns that the learned policy will see at test time.

\textbf{MILP instance generation.}
General generators can match coarse statistics such as density or degree profiles \cite{bowly2019stress}. Learned graph generators modify components of the bipartite graph \cite{geng2023g2milp}. Block-structured generators build libraries from reordered coefficient matrices and then apply block operations \cite{liu2024milpstudio}. In these approaches, the reusable unit is still not explicit. \method treats it as a graph-local module plus its coupling interface.

\textbf{MILP structure and decomposition.}
Classical decomposition separates local subproblems from global coupling constraints. Dantzig-Wolfe and Benders-type methods are standard examples \cite{dantzig1960decomposition,benders1962partitioning,geoffrion1972generalized}, and reformulation methods use similar ideas to expose integer-program structure \cite{vanderbeck2010reformulation}. We use this vocabulary for generation rather than for solving. In \method, master constraints and boundary variables record how a reusable local module reconnects to the full instance.

\section{Preliminaries}

\subsection{MILP Instance Generation}

A MILP instance consists of a sparse linear objective, linear constraints, bounds, and integrality restrictions. In this paper, instance generation starts from source instances in a target family and produces additional instances for the same family. The generated data is judged by whether it stays close to the source family in graph statistics and feasibility behavior, and by whether it helps downstream Predict-and-Search training.

\subsection{MILP Graph Representation}

We consider a MILP instance
\begin{equation}
\begin{aligned}
\min_x \quad & c^\top x\\
\mathrm{s.t.}\quad
& a_i^\top x \ \bowtie_i\ b_i,\quad i=1,\ldots,m,\\
& l \le x \le u,\quad x_j \in \mathbb{Z}\ \forall j\in I,
\end{aligned}
\label{eq:milp}
\end{equation}
where $A\in \mathbb{R}^{m\times n}$ is sparse, $a_i^\top$ is row $i$ of $A$, and $\bowtie_i\in\{\le,=,\ge\}$ is the constraint sense. We represent the instance as a weighted bipartite graph $G=(C\cup V,E)$, where $C=\{c_1,\ldots,c_m\}$ denotes constraint nodes and $V=\{v_1,\ldots,v_n\}$ denotes variable nodes. An edge $(c_i,v_j)\in E$ exists if and only if $A_{ij}\ne 0$, with edge attribute $A_{ij}$. Constraint metadata includes the right-hand side and constraint sense; variable metadata includes objective coefficient, lower and upper bounds, and variable type.

The graph is invariant to row and column permutations of the coefficient matrix up to node renaming. A row permutation only relabels constraint nodes, and a column permutation only relabels variable nodes. We use the bipartite graph as the structural object for block-unit discovery and recombination.

\subsection{Local Modules and Interfaces}

A structured MILP often contains local groups of constraints and variables that meet the rest of the model through a smaller set of couplings. We use the terms local module and interface for generation, not for a solver decomposition. A local module is the part we may store and replace. The interface records how that part reconnects to the rest of the instance.

\section{Method}

\subsection{Overview}

Figure~\ref{fig:method_overview} shows the pipeline. \method starts by finding coupling nodes in the bipartite graph and decomposing what remains. Each residual component is stored with the interface slices that touch it. Generation is a constrained replacement step: a target block unit can be replaced only by a source unit with the same local shape, interface dimensions, and metadata signature.

\begin{figure*}[t]
\centering
\includegraphics[width=0.92\textwidth]{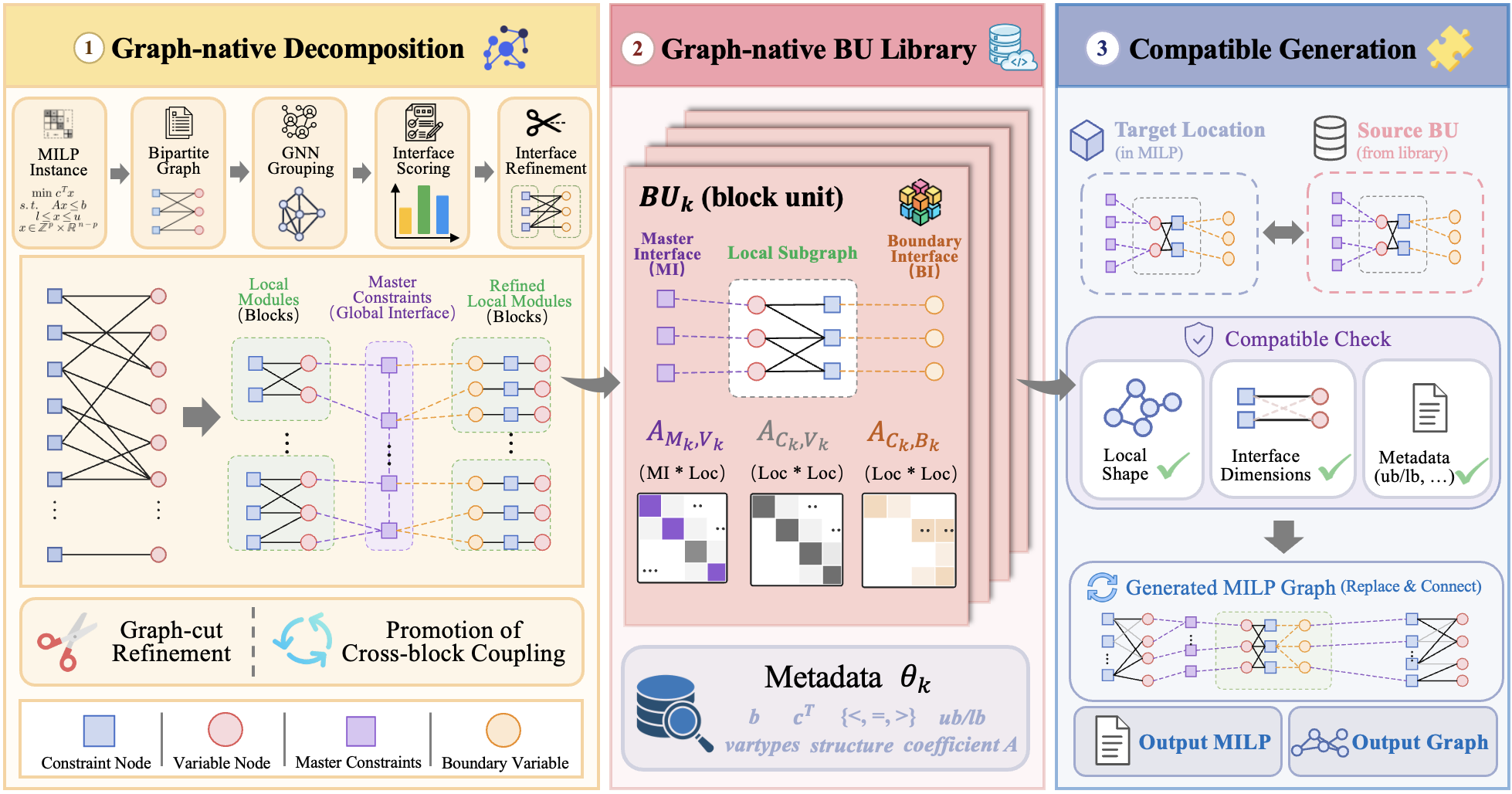}
\caption{\method overview. The three stages are graph-native decomposition, graph-native BU library construction, and compatible generation.}
\label{fig:method_overview}
\end{figure*}

\subsection{Graph-Native Decomposition}

\method first decomposes the MILP bipartite graph into local modules and coupling interfaces. Let $M\subseteq C$ denote master constraints and $B\subseteq V$ denote boundary variables. Removing these interface nodes yields the residual graph
\begin{equation}
G_{\mathrm{res}} = G[(C\setminus M)\cup (V\setminus B)].
\end{equation}
Let $\{(C_k,V_k)\}_{k=1}^{K}$ denote the residual components after connected-component decomposition and optional graph-cut refinement.
For each residual component, the relevant interface is the subset that actually touches that component:
\begin{align}
M_k &= \{m\in M:\exists v\in V_k,\ (m,v)\in E\},\\
B_k &= \{b\in B:\exists c\in C_k,\ (c,b)\in E\}.
\end{align}
The three coefficient slices needed to reinsert the component are grouped as
\begin{equation}
\mathcal{A}_k =
\big(A_{C_k,V_k}, A_{M_k,V_k}, A_{C_k,B_k}\big).
\end{equation}

\paragraph{Interface detection.}
\method starts from an initial grouping of constraints and variables. This grouping can be obtained by graph embeddings followed by clustering; subsequent stages only require group labels
\begin{equation}
g_C:C\rightarrow \mathcal{R}_C,
\qquad
g_V:V\rightarrow \mathcal{R}_V .
\end{equation}
For a constraint $c_i$, let $N(c_i)$ be its neighboring variables. The distribution of its neighbors over variable groups is
\begin{equation}
p_i(r)=
\frac{
|\{v_j\in N(c_i): g_V(v_j)=r\}|
}{
|N(c_i)|
}.
\end{equation}
We define
\begin{align}
\mathrm{span}(c_i) &= |\{r:p_i(r)>0\}|,\\
\mathrm{ent}(c_i) &=
-\frac{1}{\log \mathrm{span}(c_i)}
\sum_{r:p_i(r)>0}p_i(r)\log p_i(r),\\
\mathrm{deg}(c_i) &= |N(c_i)|.
\end{align}
When $\mathrm{span}(c_i)=1$, we set $\mathrm{ent}(c_i)=0$. Constraints are ranked lexicographically by $(\mathrm{span},\mathrm{ent},\mathrm{deg})$; high-ranking constraints are treated as candidate master constraints because they couple multiple variable groups. Boundary variables are detected symmetrically by computing the span, entropy, and degree of neighboring constraint groups.

The scores have different roles. Span finds nodes that touch multiple groups, entropy filters out nodes whose edges are concentrated in only one group, and degree avoids decisions based on very few edges. None of these scores uses row or column coordinates from the coefficient matrix.

\paragraph{Residual decomposition and promotion.}
After selecting candidate interface nodes $M$ and $B$, \method decomposes the residual graph $G_{\mathrm{res}}$. Connected components are accepted as candidate blocks. If a residual component is oversized, we recursively apply graph-cut refinement subject to balance and size constraints; in implementation we use a hybrid routine that applies Stoer--Wagner minimum cut on small components and balanced spectral refinement on larger components. These cuts are used only to propose local blocks; the corresponding edges in the original MILP remain represented through interfaces.

After residual decomposition, \method checks the original graph for cross-block coupling. Let $\beta_C(c)$ and $\beta_V(v)$ denote the current local-block labels of non-interface constraints and variables. The violated edge set is
\begin{equation}
E_{\times} =
\{(c,v)\in E:\beta_C(c)\ne \beta_V(v),\ c\notin M,\ v\notin B\}.
\end{equation}
While $E_{\times}$ is nonempty, \method promotes at least one endpoint of each violated edge into the interface set, using the span/entropy/degree score to choose the more coupling-like endpoint and a fixed constraint-side tie break. It then recomputes the residual blocks and repeats this detect-promote step until no non-interface cross-block edge remains. Graph-cut edges are represented as interfaces instead of being hidden inside stale residual components.

\begin{proposition}[Interface separation]
After cross-block promotion terminates, every edge in the original graph that connects two different local blocks is incident to at least one interface node in $M\cup B$.
\end{proposition}

\begin{proof}
Let
\begin{equation}
\begin{aligned}
I^{(t)} &= M^{(t)}\cup B^{(t)},\\
E_{\times}^{(t)}
&= \{(c,v)\in E:
\beta_C^{(t)}(c)\ne\beta_V^{(t)}(v),\\
&\quad c\notin I^{(t)},\ v\notin I^{(t)}\}.
\end{aligned}
\end{equation}
The promotion step satisfies
\begin{equation}
\forall (c,v)\in E_{\times}^{(t)},\qquad
\{c,v\}\cap I^{(t+1)}\ne\emptyset .
\end{equation}
Since $I^{(t)}\subseteq I^{(t+1)}$ and $|C\cup V|<\infty$, the process reaches a fixed point. At termination,
\begin{equation}
\begin{aligned}
E_{\times}^{(*)}
&= \{(c,v)\in E:
\beta_C^{(*)}(c)\ne\beta_V^{(*)}(v),\\
&\quad c\notin I^{(*)},\ v\notin I^{(*)}\}
=\emptyset .
\end{aligned}
\end{equation}
Therefore
\begin{equation}
\beta_C^{(*)}(c)\ne\beta_V^{(*)}(v)
\Longrightarrow
\{c,v\}\cap (M^{(*)}\cup B^{(*)})\ne\emptyset ,
\end{equation}
which is the claimed interface separation property.
\end{proof}

\subsection{Graph-Native BU Library}

The decomposition stage produces block units that can be stored and reused. For a residual component $(C_k,V_k)$, the corresponding block unit is
\begin{equation}
\mathrm{BU}_k =
\big(C_k,V_k,M_k,B_k,\mathcal{A}_k,\theta_k\big),
\end{equation}
where $\theta_k$ stores the local metadata associated with $C_k$ and $V_k$, including right-hand sides, constraint senses, objective coefficients, bounds, and variable types.

The internal slice $A_{C_k,V_k}$ describes the local module, while $A_{M_k,V_k}$ and $A_{C_k,B_k}$ record its interaction with adjacent master constraints and boundary variables. A block unit contains both the local subproblem and the interface terms needed to reinsert it into a full MILP, without carrying unrelated global interface nodes.

\subsection{Compatible Generation}

The extracted block units form a library $\mathcal{L}=\{\mathrm{BU}_1,\ldots,\mathrm{BU}_N\}$. To generate a new instance, \method decomposes a target instance and tries to replace one of its block units with a compatible source unit from $\mathcal{L}$. Compatibility is checked before any coefficient slice is copied. We define the shape signature
\begin{equation}
\begin{aligned}
\Phi(\mathrm{BU}) =
\big(&\mathrm{shape}(A_{C,V}),\\
&\mathrm{shape}(A_{M,V}),\\
&\mathrm{shape}(A_{C,B})\big),
\end{aligned}
\end{equation}
and require $\Phi(\mathrm{BU}_t)=\Phi(\mathrm{BU}_s)$.
The target keeps its variable types and bounds. These fields are not treated as style metadata; changing them can change the feasible region even when the graph shape is unchanged. For a modification ratio $\eta$, experiments set the replacement budget to $q=\lfloor\eta |\mathcal{T}|\rfloor$, where $\mathcal{T}$ is the target block-unit set. The generator does not solve the sufficient condition in Proposition~\ref{prop:feasible_replacement}; it keeps the quantities that make that condition well defined.

Table~\ref{tab:replacement_policy} summarizes the replacement policy. Source and target rows and columns are aligned by the deterministic order stored during extraction: local constraints, local variables, master-interface rows, and boundary-interface columns are sorted by extraction order with stable original-index tie breaks. If any required metadata signature or interface shape is incompatible, the selected target unit is left unchanged.

\begin{table}[t]
\centering
\small
\setlength{\tabcolsep}{4pt}
\renewcommand{\arraystretch}{1.08}
\caption{Compatible replacement policy for a target block unit and a sampled source block unit.}
\label{tab:replacement_policy}
\begin{tabularx}{\columnwidth}{@{}lX@{}}
\toprule
\textbf{Component} & \textbf{Policy} \\
\midrule
$A_{C_k,V_k}$ & copied from source local slice \\
$A_{M_k,V_k}$ & copied from source master-interface slice \\
$A_{C_k,B_k}$ & copied from source boundary-interface slice \\
$b_{C_k}$, $c_{V_k}$ & copied from source after sense check \\
Sense sequence & must match target sequence \\
Variable type, bounds & preserved from target \\
Interface shape & must match exactly \\
Alignment & stable extraction order \\
\bottomrule
\end{tabularx}
\end{table}

\begin{proposition}[Feasible replacement under interface slack]
\label{prop:feasible_replacement}
Write all affected rows in $\le$ form. Let $x$ be a feasible solution of the original MILP, and suppose only block $k$ is replaced while all variables outside $V_k$ are kept fixed. Let $R_k$ be the set of rows that are neither local rows $C_k$ nor affected master rows $M_k$. Assume rows in $R_k$ and the coefficients on outside variables are unchanged. Define the residual master capacity
\begin{equation}
\rho_{M_k}
=
b_{M_k}-A^{old}_{M_k,-V_k}x_{-V_k}.
\end{equation}
If there exists a local assignment $\tilde{x}_{V_k}$ such that
\begin{align}
l_{V_k}\le \tilde{x}_{V_k}\le u_{V_k},\qquad
&\tilde{x}_j\in\mathbb{Z}\quad \forall j\in I\cap V_k, \label{eq:domain_condition}\\
A^{new}_{C_k,V_k}\tilde{x}_{V_k}
+A^{new}_{C_k,B_k}x_{B_k}
&\le b^{new}_{C_k}, \label{eq:local_condition}\\
A^{new}_{M_k,V_k}\tilde{x}_{V_k}
&\le \rho_{M_k}, \label{eq:master_condition}
\end{align}
then $\tilde{x}=(x_{-V_k},\tilde{x}_{V_k})$ is feasible for the replaced MILP.
\end{proposition}

\begin{proof}
The replaced solution satisfies the domain constraints by
\begin{equation}
\tilde{x}_{-V_k}=x_{-V_k},
\qquad
l_{V_k}\le \tilde{x}_{V_k}\le u_{V_k},
\qquad
\tilde{x}_j\in\mathbb{Z}\ \forall j\in I\cap V_k .
\end{equation}
The local rows are feasible directly from Eq.~\eqref{eq:local_condition}:
\begin{equation}
A^{new}_{C_k,:}\tilde{x}
=
A^{new}_{C_k,V_k}\tilde{x}_{V_k}
+A^{new}_{C_k,B_k}x_{B_k}
\le b^{new}_{C_k}.
\end{equation}
For affected master rows, Eq.~\eqref{eq:master_condition} gives
\begin{equation}
\begin{aligned}
A^{new}_{M_k,:}\tilde{x}
&=
A^{new}_{M_k,V_k}\tilde{x}_{V_k}
+A^{old}_{M_k,-V_k}x_{-V_k}\\
&\le
\rho_{M_k}
+A^{old}_{M_k,-V_k}x_{-V_k}
= b_{M_k}.
\end{aligned}
\end{equation}
For all remaining rows $R_k$, the row coefficients and variables are unchanged, hence
\begin{equation}
A^{new}_{R_k,:}\tilde{x}
=
A^{old}_{R_k,:}x
\le b_{R_k}.
\end{equation}
Combining the three row groups,
\begin{equation}
\begin{bmatrix}
A^{new}_{C_k,:}\\
A^{new}_{M_k,:}\\
A^{new}_{R_k,:}
\end{bmatrix}
\tilde{x}
\le
\begin{bmatrix}
b^{new}_{C_k}\\
b_{M_k}\\
b_{R_k}
\end{bmatrix},
\end{equation}
and the domain constraints also hold. Therefore $\tilde{x}$ is feasible for the replaced MILP.
\end{proof}

\begin{algorithm}[t]
\caption{\method block-unit generation}
\label{alg:graphbu}
\small
\begin{algorithmic}[1]
\REQUIRE source set $\mathcal{D}$; target MILP $P$; budget $q$
\ENSURE generated MILP $\hat P$
\STATE $\mathcal{L}\leftarrow\emptyset$
\FOR{each source instance $P_s\in\mathcal{D}$}
    \STATE $\mathcal{L}\leftarrow\mathcal{L}\cup\operatorname{ExtractBU}(P_s)$
\ENDFOR
\STATE $\hat P\leftarrow P$
\STATE $\mathcal{T}\leftarrow\operatorname{ExtractBU}(\hat P)$
\FOR{$r=1,\ldots,q$}
    \STATE sample a target unit $\mathrm{BU}_t\in\mathcal{T}$
    \STATE $\mathcal{C}\leftarrow$ source units compatible with $\mathrm{BU}_t$
    \IF{$\mathcal{C}=\emptyset$}
        \STATE continue
    \ENDIF
    \STATE sample $\mathrm{BU}_s\in\mathcal{C}$
    \STATE replace local and interface slices using Table~\ref{tab:replacement_policy}
    \STATE preserve target variable types, bounds, and interface dimensions
\ENDFOR
\STATE \textbf{return} $\hat P$
\end{algorithmic}
\end{algorithm}

\noindent\textsc{ExtractBU}$(P)$ denotes the same extraction routine for source and target instances. It builds the bipartite graph, selects initial interface nodes from grouped neighborhoods, runs residual decomposition with promotion, and returns block-local units with their adjacent interfaces.

\subsection{Structural Properties}

The next property states the row-column order invariance implied by the graph representation. Let two MILP instances be equivalent if one can be obtained from the other by permuting constraint rows and variable columns.

\begin{proposition}[Permutation invariance]
Suppose two MILP instances differ only by row and column permutations. If the grouping module is equivariant to the induced graph isomorphism up to group-label renaming, then \method produces equivalent interface sets, residual blocks, and block units up to node renaming.
\end{proposition}

\begin{proof}
Let $\phi=(\phi_C,\phi_V)$ be the graph isomorphism induced by the row
and column permutations, and let $\rho_C,\rho_V$ be the corresponding
renamings of constraint and variable group labels. Equivariance of the
grouping module means
\begin{equation}
g'_C(\phi_C(c))=\rho_C(g_C(c)),
\qquad
g'_V(\phi_V(v))=\rho_V(g_V(v)).
\end{equation}
For any constraint $c$, the neighbor distribution is therefore preserved
up to label renaming:
\begin{equation}
p'_{\phi_C(c)}(\rho_V(r)) = p_c(r).
\end{equation}
Thus
\begin{equation}
(\mathrm{span}',\mathrm{ent}',\mathrm{deg}')(\phi_C(c))
=
(\mathrm{span},\mathrm{ent},\mathrm{deg})(c),
\end{equation}
and the same argument holds symmetrically for variables. The
lexicographic ranking consequently selects corresponding interfaces,
$M'=\phi_C(M)$ and $B'=\phi_V(B)$. Removing these nodes gives
\begin{equation}
G'_{\mathrm{res}}
=
\phi(G_{\mathrm{res}}),
\end{equation}
so connected components and graph-cut refinements correspond under
$\phi$. The promotion rule depends only on adjacency and block labels,
which are preserved by $\phi$, and therefore also commutes with the
permutation. Finally, block-unit extraction applies the same node
renaming to the internal and interface slices. Hence the interface
sets, residual blocks, and block units are equivalent up to node
renaming.
\end{proof}

These propositions give structural guarantees without asserting a unique semantic decomposition or universal hardness preservation. Cross-block coupling is represented through explicit interfaces. Interface-aware replacement keeps the quantities needed by the sufficient feasibility condition. Under equivariant grouping, block discovery is independent of arbitrary row-column order.

\section{Experiments}

\begin{table*}[!t]
\centering
\scriptsize
\setlength{\tabcolsep}{2.7pt}
\renewcommand{\arraystretch}{1.02}
\caption{
Generation quality at $\eta=0.05$. Each dataset reports graph-statistical similarity, feasible ratio, and average solve time in seconds. Higher similarity and feasible ratio are better. Solve time should be read only when the generated distribution remains faithful and feasible.
}
\label{tab:generation_quality}
\resizebox{\textwidth}{!}{
\begin{tabular}{lccc ccc ccc ccc}
\toprule
\multirow{2}{*}{\textbf{Method}}
& \multicolumn{3}{c}{\textbf{CA}}
& \multicolumn{3}{c}{\textbf{FA}}
& \multicolumn{3}{c}{\textbf{IP}}
& \multicolumn{3}{c}{\textbf{WA}} \\
\cmidrule(lr){2-4}\cmidrule(lr){5-7}\cmidrule(lr){8-10}\cmidrule(lr){11-13}
& Similarity $\uparrow$ & Feas. $\uparrow$ & Time
& Similarity $\uparrow$ & Feas. $\uparrow$ & Time
& Similarity $\uparrow$ & Feas. $\uparrow$ & Time
& Similarity $\uparrow$ & Feas. $\uparrow$ & Time \\
\midrule
Random
& 0.593 & 100.0\% & 1000.00
& 0.000 & 0.0\% & 0.03
& 0.248 & 98.0\% & 0.15
& 0.040 & 0.0\% & 0.15 \\
G2MILP
& 0.484 & 100.0\% & 0.69
& 0.092 & 1.8\% & 0.01
& 0.352 & 100.0\% & 0.14
& 0.484 & 0.0\% & 0.01 \\
MILP-StuDio
& 0.936 & 100.0\% & 1000.00
& 0.723 & 100.0\% & 2.07
& 0.687 & 68.0\% & 340.01
& 0.910 & 92.0\% & 270.39 \\
\rowcolor{ourhighlight}
\textbf{\method}
& \textbf{0.989} & 100.0\% & 1000.00
& \textbf{0.813} & 100.0\% & 2.78
& \textbf{0.940} & 88.0\% & 715.37
& \textbf{0.995} & \textbf{98.7\%} & 871.51 \\
\bottomrule
\end{tabular}
}
\end{table*}

\subsection{Experimental Setup}

We evaluate whether the generated instances remain usable as data from the target MILP family.

\begin{itemize}
\item \textbf{Datasets.} We use four MILP families: combinatorial auctions (CA), capacitated facility location (FA), item placement (IP), and workload appointment (WA). Appendix Table~\ref{tab:app_dataset_scale} reports their original scales.
\item \textbf{Metrics.} Generation quality is measured by graph-statistical similarity, feasible ratio, and average Gurobi solve time. For a statistic $z$, let $d_z\in[0,1]$ be its normalized distributional distance after clipping. We report
\begin{equation}
\mathrm{Similarity}=|\mathcal{Z}|^{-1}\sum_{z\in\mathcal{Z}}(1-d_z),
\end{equation}
so larger values mean closer agreement with the original family under size, sparsity, degree, coefficient, clustering, and modularity statistics. Downstream utility is measured by Predict-and-Search (PS) performance on held-out original instances.
\item \textbf{Baselines.} We compare against random replacement, G2MILP, and MILP-StuDio. These cover unstructured replacement, learned graph-level generation, and block generation from reordered coefficient matrices.
\item \textbf{Experimental details.} We evaluate modification ratios $\eta\in\{0.01,0.05,0.10\}$ and use a 1000-second Gurobi validation limit for all methods. Each generation-quality setting contains 100 CA, 50 FA, 100 IP, and 150 WA generated instances. Downstream PS models are trained on generated instances and evaluated on 40 held-out original instances per family; held-out instances are never used as generation sources or targets.
\end{itemize}

\subsection{Overall Performance}

\paragraph{Similarity and hardness.}
Table~\ref{tab:generation_quality} reports generation quality at $\eta=0.05$. Similarity averages $1-d_z$ over the structural statistics $z$, where $d_z$ is a clipped normalized distance between generated and original distributions. A larger value means less drift from the target family. \method has the highest similarity on all four datasets: 0.989 on CA, 0.813 on FA, 0.940 on IP, and 0.995 on WA. Feasibility is also stable on most families. CA and FA remain fully feasible, WA reaches 98.7\%, and IP improves over the block-structured baseline (88.0\% versus 68.0\%). Solve time is reported as a hardness signal, but it should not be read alone. A fast infeasible or low-similarity instance is usually a degenerate sample, not an easy case worth training on.

\begin{figure}[!tbp]
\centering
\includegraphics[width=\columnwidth]{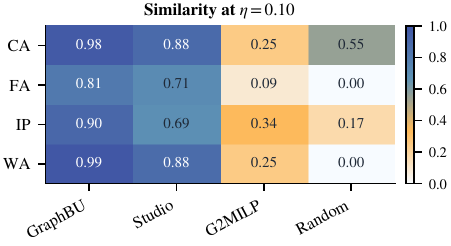}
\caption{Graph-statistical similarity at $\eta=0.10$. Blue cells indicate stronger agreement with the original family.}
\label{fig:similarity_heatmap}
\end{figure}

Figure~\ref{fig:similarity_heatmap} repeats the comparison at $\eta=0.10$. \method is still closest to the original distribution on all four datasets. WA is the clearest case: the decomposition finds many replaceable local units with compatible interfaces, so the replacement changes local slices without moving the global graph statistics much. The generated instances are not copies, since degree and coefficient moments still shift after replacement. Similarity matters here because learned solver policies depend on staying in the same structural regime. We still report feasibility, solve time, and PS performance because distributional closeness alone is too weak.

\paragraph{Downstream Predict-and-Search.}
Table~\ref{tab:ps_main} tests the generated data in a downstream PS setting. CA, IP, and WA hit the time limit, so we compare final primal-dual gaps. FA is solved by all methods, so runtime is the relevant metric. Relative to the PS baseline, PS+\method reduces the CA gap from 7.8165 to 7.7394, the IP gap from 74.8606 to 73.2195, and the WA gap from 0.4900 to 0.4724. On FA, the average runtime drops from 4.82s to 3.61s while all 40 instances remain solved. The gains are not large on every family, but they appear exactly where the generated data remains structurally close to the source family.

\begin{table*}[!tbp]
\centering
\small
\setlength{\tabcolsep}{4.2pt}
\renewcommand{\arraystretch}{1.06}
\caption{
Downstream Predict-and-Search performance on original test instances. CA, IP, and WA use gap reduction as gain because all methods hit the time limit. FA uses runtime reduction because all methods solve all instances. Objective values follow the direction of each dataset; gap/time and gain are the primary comparison metrics. Best PS-variant results are in bold.
}
\label{tab:ps_main}
\resizebox{\textwidth}{!}{
\begin{tabular}{lccc ccc ccc ccc}
\toprule
\multirow{2}{*}{\textbf{Method}}
& \multicolumn{3}{c}{\textbf{CA} \;(\textit{timeout: 40/40})}
& \multicolumn{3}{c}{\textbf{FA} \;(\textit{optimal: 40/40})}
& \multicolumn{3}{c}{\textbf{IP} \;(\textit{timeout: 40/40})}
& \multicolumn{3}{c}{\textbf{WA} \;(\textit{timeout: 40/40})} \\
\cmidrule(lr){2-4}\cmidrule(lr){5-7}\cmidrule(lr){8-10}\cmidrule(lr){11-13}
& \textbf{Obj.} $\uparrow$ & \textbf{Gap} $\downarrow$ & \textbf{Gain} $\uparrow$
& \textbf{Time} $\downarrow$ & \textbf{Gap} $\downarrow$ & \textbf{Gain} $\uparrow$
& \textbf{Obj.} $\downarrow$ & \textbf{Gap} $\downarrow$ & \textbf{Gain} $\uparrow$
& \textbf{Obj.} $\downarrow$ & \textbf{Gap} $\downarrow$ & \textbf{Gain} $\uparrow$ \\
\midrule
Gurobi baseline
& 97416.6760 & 7.9177 & --
& 5.13 & 0.00173 & --
& 11.1934 & 61.1536 & --
& 699.300 & 0.4864 & -- \\
PS baseline
& 97467.1819 & 7.8165 & 0.00\%
& 4.82 & \textbf{0.00144} & 0.00\%
& 10.9523 & 74.8606 & 0.00\%
& 699.275 & 0.4900 & 0.00\% \\
PS+MILP-StuDio
& 97384.0230 & 8.0029 & -2.38\%
& 5.02 & 0.00167 & -4.15\%
& 11.0653 & 73.9103 & 1.27\%
& 699.325 & 0.4905 & -0.10\% \\
\rowcolor{ourhighlight}
\textbf{PS+\method}
& \textbf{97529.0606} & \textbf{7.7394} & \textbf{0.99\%}
& \textbf{3.61} & \textbf{0.00144} & \textbf{25.10\%}
& \textbf{10.7780} & \textbf{73.2195} & \textbf{2.19\%}
& \textbf{699.250} & \textbf{0.4724} & \textbf{3.59\%} \\
\bottomrule
\end{tabular}
}
\end{table*}

\begin{figure}[!tbp]
\centering
\includegraphics[width=\columnwidth]{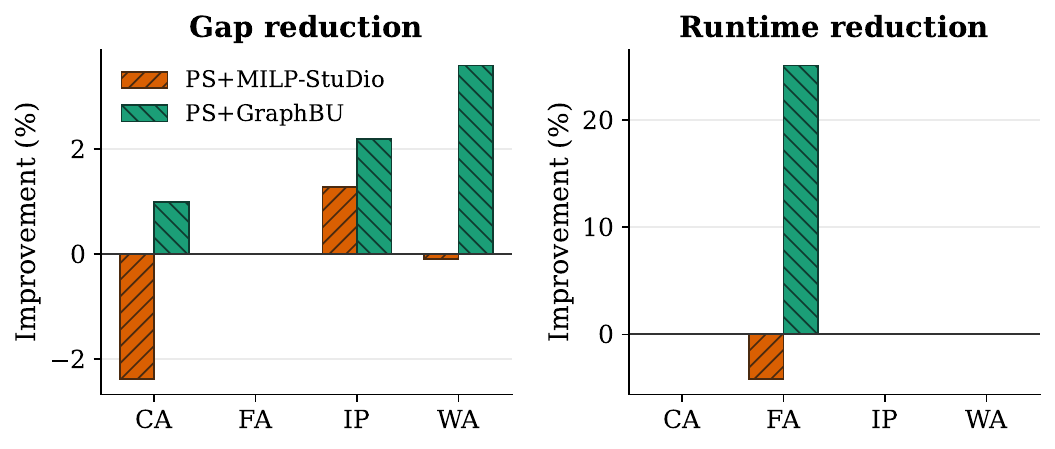}
\caption{Relative improvement over PS baseline. CA/IP/WA use gap reduction; FA uses runtime reduction. Non-primary dataset-metric pairs are shown as zero.}
\label{fig:ps_improvement}
\end{figure}

\subsection{Block-Pool Statistics}

Figure~\ref{fig:block_pool_diagnostics} explains when compatible replacement has room to operate. The x-axis counts extracted block units on a log scale. Compatibility measures how often a target unit can find a source unit with the same shape and metadata. Bubble size shows whether one block dominates the library, and edge width records how much coupling is exposed as interfaces. The datasets behave differently. IP has a small but highly compatible pool, which helps explain its improvement over the block-structured baseline. WA has many reusable units and enough matching interfaces; this is consistent with its high similarity. CA has fewer compatible matches, so the conservative replacement rule rejects more substitutions.

\begin{figure}[!tbp]
\centering
\includegraphics[width=\columnwidth]{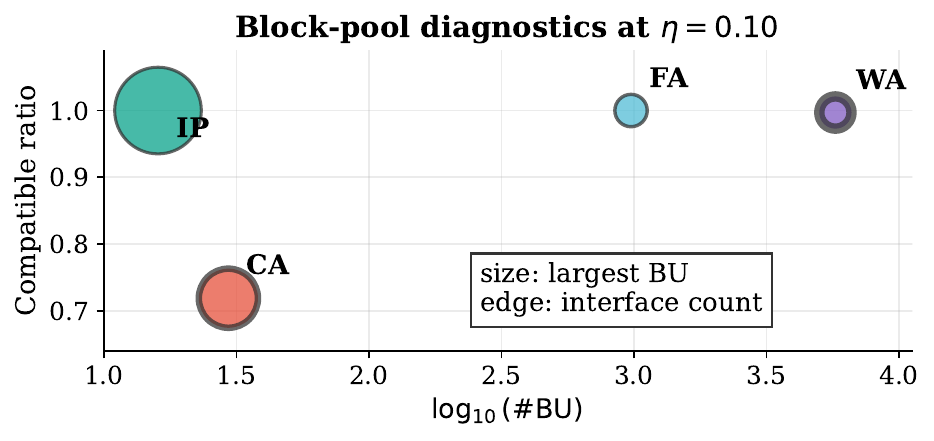}
\caption{Block-pool statistics at $\eta=0.10$. The plot links the decomposition output to compatible generation: reusable block units matter when their interface signatures admit replacement.}
\label{fig:block_pool_diagnostics}
\end{figure}

\section{Conclusion}

\method is a MILP instance generator built around graph-native block units. Each unit contains a local module and the interface through which it connects to the rest of the instance. Generation is not tied to a fixed row-column ordering, and replacement depends on explicit compatibility checks.

The structural results are intentionally modest. Promotion makes cross-block coupling explicit. The replacement rule keeps the quantities needed by a sufficient feasibility condition. The graph representation removes dependence on row-column order under equivariant grouping. In experiments, \method stays close to the source families in graph-statistical similarity and preserves feasibility on most datasets. The downstream PS results suggest that this structure matters for solver training, beyond matching summary statistics.

\bibliography{references}

@article{miplib2017,
  title = {{MIPLIB} 2017: Data-Driven Compilation of the 6th Mixed-Integer Programming Library},
  author = {Gleixner, Ambros and Hendel, Gregor and Gamrath, Gerald and Achterberg, Tobias and Bastubbe, Michael and Berthold, Timo and Christophel, Philipp and Jarck, Kati and Koch, Thorsten and Linderoth, Jeff and others},
  journal = {Mathematical Programming Computation},
  volume = {13},
  number = {3},
  pages = {443--490},
  year = {2021},
  publisher = {Springer}
}

@article{gasse2022ml4co,
  title = {The Machine Learning for Combinatorial Optimization Competition ({ML4CO}): Results and Insights},
  author = {Gasse, Maxime and Cappart, Quentin and Charfreitag, Jonas and Charlin, Laurent and Ch{\'e}telat, Didier and Chmiela, Antonia and Dumouchelle, Justin and Gleixner, Ambros M. and Kazachkov, Aleksandr M. and Khalil, Elias B. and Lichocki, Pawel and Lodi, Andrea and Lubin, Miles and Maddison, Chris J. and Morris, Christopher and Papageorgiou, Dimitri J. and Parjadis, Augustin and Pokutta, Sebastian and Prouvost, Antoine and Scavuzzo, Lara and Zarpellon, Giulia and Yang, Linxin and Lai, Sha and Wang, Akang and Luo, Xiaodong and Zhou, Xiang and Huang, Haohan and Shao, Sheng Cheng and Zhu, Yuanming and Zhang, Dong and Quan, Tao and Cao, Zixuan and Xu, Yang and Huang, Zhewei and Zhou, Shuchang and Chen, Binbin and He, Minggui and Hao, Hao and Zhang, Zhiyu and An, Zhiwu and Mao, Kun},
  journal = {arXiv preprint arXiv:2203.02433},
  year = {2022}
}

@article{gasse2019exact,
  title = {Exact Combinatorial Optimization with Graph Convolutional Neural Networks},
  author = {Gasse, Maxime and Ch{\'e}telat, Didier and Ferroni, Nicola and Charlin, Laurent and Lodi, Andrea},
  journal = {arXiv preprint arXiv:1906.01629},
  year = {2019}
}

@article{prouvost2020ecole,
  title = {Ecole: A Gym-Like Library for Machine Learning in Combinatorial Optimization Solvers},
  author = {Prouvost, Antoine and Dumouchelle, Justin and Scavuzzo, Lara and Gasse, Maxime and Ch{\'e}telat, Didier and Lodi, Andrea},
  journal = {arXiv preprint arXiv:2011.06069},
  year = {2020}
}

@article{bengio2021mlco,
  title = {Machine Learning for Combinatorial Optimization: A Methodological Tour d'Horizon},
  author = {Bengio, Yoshua and Lodi, Andrea and Prouvost, Antoine},
  journal = {European Journal of Operational Research},
  volume = {290},
  number = {2},
  pages = {405--421},
  year = {2021}
}

@inproceedings{khalil2016learning,
  title = {Learning to Branch in Mixed Integer Programming},
  author = {Khalil, Elias Boutros and Le Bodic, Pierre and Song, Le and Nemhauser, George L. and Dilkina, Bistra},
  booktitle = {Proceedings of the AAAI Conference on Artificial Intelligence},
  year = {2016}
}

@article{nair2020neuralmip,
  title = {Solving Mixed Integer Programs Using Neural Networks},
  author = {Nair, Vinod and Bartunov, Sergey and Gimeno, Felix and von Glehn, Tamara and Lichocki, Pawel and Lobov, Ivan and O'Donoghue, Brendan and Sonnerat, Nicolas and Tjandraatmadja, Christian and Wang, Pengming and Addanki, Ravichandra and Hapuarachchi, Theophane and Keck, Thomas and Keeling, James and Kohli, Pushmeet and Ktena, Ira and Li, Yujia and Vinyals, Oriol and Zwols, Yori},
  journal = {arXiv preprint arXiv:2012.13349},
  year = {2020}
}

@inproceedings{han2023predictsearch,
  title = {A {GNN}-Guided Predict-and-Search Framework for Mixed-Integer Linear Programming},
  author = {Han, Qingyu and Yang, Linxin and Chen, Qian and Zhou, Xiang and Zhang, Dong and Wang, Akang and Sun, Ruoyu and Luo, Xiaodong},
  booktitle = {International Conference on Learning Representations},
  year = {2023}
}

@phdthesis{bowly2019stress,
  title = {Stress Testing Mixed Integer Programming Solvers through New Test Instance Generation Methods},
  author = {Bowly, Simon},
  school = {Monash University},
  year = {2019}
}

@inproceedings{geng2023g2milp,
  title = {A Deep Instance Generative Framework for {MILP} Solvers Under Limited Data Availability},
  author = {Geng, Zijie and Li, Xijun and Wang, Jie and Li, Xiao and Zhang, Yongdong and Wu, Feng},
  booktitle = {Advances in Neural Information Processing Systems},
  year = {2023}
}

@inproceedings{liu2024milpstudio,
  title = {{MILP-StuDio}: {MILP} Instance Generation via Block Structure Decomposition},
  author = {Geng, Zijie and Kuang, Yufei and Li, Bin and Li, Xijun and Liu, Haoyang and Wang, Jie and Wu, Feng and Zhang, Wanbo and Zhang, Yongdong},
  booktitle = {Advances in Neural Information Processing Systems},
  year = {2024}
}

@article{dantzig1960decomposition,
  title = {Decomposition Principle for Linear Programs},
  author = {Dantzig, George B. and Wolfe, Philip},
  journal = {Operations Research},
  volume = {8},
  number = {1},
  pages = {101--111},
  year = {1960}
}

@article{benders1962partitioning,
  title = {Partitioning Procedures for Solving Mixed-Variables Programming Problems},
  author = {Benders, Jacques F.},
  journal = {Computational Management Science},
  volume = {2},
  number = {1},
  pages = {3--19},
  year = {2005}
}

@article{geoffrion1972generalized,
  title = {Generalized Benders Decomposition},
  author = {Geoffrion, Arthur M.},
  journal = {Journal of Optimization Theory and Applications},
  volume = {10},
  number = {4},
  pages = {237--260},
  year = {1972}
}

@incollection{vanderbeck2010reformulation,
  title = {Reformulation and Decomposition of Integer Programs},
  author = {Vanderbeck, Fran{\c{c}}ois and Wolsey, Laurence A.},
  booktitle = {50 Years of Integer Programming 1958--2008},
  pages = {431--502},
  year = {2010},
  publisher = {Springer}
}

\appendix
\clearpage
\begingroup
\setlength{\textfloatsep}{7pt plus 1pt minus 2pt}
\setlength{\floatsep}{6pt plus 1pt minus 2pt}
\setlength{\intextsep}{6pt plus 1pt minus 2pt}
\captionsetup[table]{skip=3pt}
\makeatletter
\setlength{\@fptop}{0pt}
\setlength{\@fpsep}{8pt plus 1pt minus 1pt}
\setlength{\@fpbot}{0pt plus 1fil}
\setlength{\@dblfptop}{0pt}
\setlength{\@dblfpsep}{8pt plus 1pt minus 1pt}
\setlength{\@dblfpbot}{0pt plus 1fil}
\makeatother

\section{Additional Method Details}

\subsection{Graph Construction and Metadata}

For each MILP instance, \method constructs a weighted bipartite graph
$G=(C\cup V,E)$ from the nonzero pattern of the constraint matrix. The
constraint side $C$ contains one node per row of $A$, and the variable
side $V$ contains one node per decision variable. Each nonzero
$A_{ij}$ induces an edge $(c_i,v_j)$ with edge weight $A_{ij}$.
The graph topology is used for decomposition, while MILP metadata is
stored with the corresponding nodes and edges. Constraint metadata
includes the right-hand side and constraint sense
$(\le,=,\ge)$. Variable metadata includes objective coefficient, lower
bound, upper bound, and variable type. Edge metadata stores the
coefficient value.

This separation matters during generation. The adjacency pattern
determines which constraints and variables are coupled, while metadata
controls feasibility-critical details that should not be changed
without compatibility checks. In particular, \method preserves variable types and
bounds during compatible replacement, since changing integrality or
domain information can alter the feasible region even when graph shape
is unchanged.

\subsection{Symmetric Interface Scores}

The main text defines span, entropy, and degree scores for constraint
nodes. Boundary variables are scored symmetrically. Let $N(v_j)$ be the
neighboring constraints of variable $v_j$, and let $g_C:C\rightarrow
\mathcal{R}_C$ denote constraint group labels. Define
\begin{equation}
q_j(r)=
\frac{
|\{c_i\in N(v_j):g_C(c_i)=r\}|
}{
|N(v_j)|
}.
\end{equation}
The variable-side coupling scores are
\begin{align}
\mathrm{span}(v_j) &= |\{r:q_j(r)>0\}|,\\
\mathrm{ent}(v_j) &=
-\frac{1}{\log \mathrm{span}(v_j)}
\sum_{r:q_j(r)>0}q_j(r)\log q_j(r),\\
\mathrm{deg}(v_j) &= |N(v_j)|.
\end{align}
When $\mathrm{span}(v_j)=1$, we set $\mathrm{ent}(v_j)=0$. Variables
with high span, entropy, and degree are candidate boundary variables
because they connect multiple constraint groups.

\subsection{Grouping Module}

\method only assumes an initial grouping of constraint and variable
nodes. This grouping can be instantiated by
spectral biclustering on the sparse coefficient structure or by graph
embeddings followed by clustering. The downstream block-unit extraction
does not require the grouping labels to correspond to semantic classes:
they are used to expose nodes whose neighborhoods span multiple groups.
The reported experiments use graph embeddings clustered by HDBSCAN.
Spectral biclustering provides an alternative grouping procedure.
The interface detector is separated from a particular grouping
method: different embeddings or domain-aware clustering schemes can be
used without changing the definition of graph-native block units or the
compatible replacement rule.

\subsection{Compatibility Checks}

For a target block unit $\mathrm{BU}_t$ and a source block unit
$\mathrm{BU}_s$, the main text requires matching internal shape and
block-local interface dimensions. The replacement rule also checks
metadata compatibility before replacement. In particular, the local
constraint-sense sequence must match, and the variable type sequence and
bound pattern of the target local variables are preserved. The source
unit contributes local coefficient slices, local right-hand sides, and
local objective coefficients only when these checks pass. If no
compatible source exists for a selected target block, that target block
is left unchanged and another target block is considered.

These checks are intentionally conservative. They do not certify
feasibility, but they rule out several avoidable failures: replacing a
binary block with a continuous block, changing the number of local
constraints incident to the block-local master interface, changing the
number of boundary-variable columns seen by a local constraint, or
mixing local rows with incompatible constraint senses.

\section{Additional Structural Discussion}

\subsection{Relation to Classical Decomposition}

Classical decomposition methods such as Dantzig-Wolfe decomposition use
local subproblems and global coupling constraints to solve optimization
problems. \method uses the same vocabulary for a different purpose:
generation rather than optimization. A graph-native block unit is not a
pricing problem, and the master constraints in \method are not used to
derive a decomposition algorithm for solving. They are interface nodes
that record how a local module connects to the rest of the MILP.

This distinction limits the claims. \method does
not claim to recover a unique economic, physical, or semantic
decomposition of an instance. It claims that, after interface promotion,
cross-block coupling is represented explicitly and can be respected
during generation.

\subsection{Expanded Proof of Interface Separation}

The main paper states that, after promotion, every cross-block edge in
the original graph is incident to an interface node. We restate the
argument with the block assignment made explicit.

Let $\pi_C^{(t)}(c_i)$ and $\pi_V^{(t)}(v_j)$ denote residual block
assignments at iteration $t$, and define
\begin{equation}
I^{(t)}=M^{(t)}\cup B^{(t)} .
\end{equation}
At each promotion iteration,
\begin{equation}
\begin{aligned}
E_{\times}^{(t)}=
\{(c_i,v_j)\in E:&\pi_C^{(t)}(c_i)\ne \pi_V^{(t)}(v_j),\\
&c_i\notin I^{(t)},\ v_j\notin I^{(t)}\}.
\end{aligned}
\end{equation}
For every violated edge, the update satisfies
\begin{equation}
\forall (c_i,v_j)\in E_{\times}^{(t)},\qquad
\{c_i,v_j\}\cap I^{(t+1)}\ne\emptyset ,
\end{equation}
with
\begin{equation}
I^{(t)}\subseteq I^{(t+1)}\subseteq C\cup V .
\end{equation}
Since $C\cup V$ is finite, the monotone sequence reaches a fixed point
$I^{(*)}$. The stopping condition is
\begin{equation}
E_{\times}^{(*)}=\emptyset .
\end{equation}
Equivalently,
\begin{equation}
\begin{aligned}
\forall (c_i,v_j)\in E,\quad
\pi_C^{(*)}(c_i)\ne \pi_V^{(*)}(v_j)
\Rightarrow\quad
\{c_i,v_j\}\cap I^{(*)}\ne\emptyset .
\end{aligned}
\end{equation}

\subsection{A Sufficient Condition for Feasibility Preservation}

This section gives a sufficient, not necessary, condition under which
replacing one block unit preserves feasibility. It explains the role of
interface-aware compatibility while avoiding a universal
feasibility claim.

For compactness, write the relevant local and master constraints in
$\le$ form; equality constraints can be represented by two inequalities,
and $\ge$ rows can be multiplied by $-1$. Consider a feasible original
solution $x$ for a MILP. Suppose block $k$ with local variables $V_k$ is
replaced, while all variables outside $V_k$ are kept fixed. Let $B_k$ be
the block-local boundary variables and $M_k$ be the block-local master
constraints. If there exists a replacement local assignment
$\tilde{x}_{V_k}$ satisfying the preserved domain constraints
\begin{equation}
l_{V_k}\le \tilde{x}_{V_k}\le u_{V_k},
\qquad
\tilde{x}_j\in\mathbb{Z}\quad \forall j\in I\cap V_k,
\end{equation}
such that
\begin{equation}
A^{new}_{C_k,V_k}\tilde{x}_{V_k}
+ A^{new}_{C_k,B_k}x_{B_k}
\le b^{new}_{C_k},
\end{equation}
and its contribution to each master constraint does not exceed the
original contribution plus remaining master slack,
\begin{equation}
A^{new}_{M_k,V_k}\tilde{x}_{V_k}
\le
A^{old}_{M_k,V_k}x_{V_k}
+ s_{M_k},
\end{equation}
where
\begin{equation}
s_{M_k} = b_{M_k} - A^{old}_{M_k,:}x,
\end{equation}
then the concatenated solution
$\tilde{x}=(x_{-V_k},\tilde{x}_{V_k})$ is feasible for the replaced
instance.

\begin{proof}
The domain condition gives
\begin{equation}
\tilde{x}\in [l,u],\qquad
\tilde{x}_j\in\mathbb{Z}\quad \forall j\in I .
\end{equation}
For local rows in $C_k$,
\begin{equation}
A^{new}_{C_k,V_k}\tilde{x}_{V_k}
+ A^{new}_{C_k,B_k}x_{B_k}
\le b^{new}_{C_k}.
\end{equation}
For master rows in $M_k$,
\begin{equation}
\begin{aligned}
A^{new}_{M_k,:}\tilde{x}
&=
A^{new}_{M_k,V_k}\tilde{x}_{V_k}
+ A^{old}_{M_k,-V_k}x_{-V_k}\\
&\le
A^{old}_{M_k,V_k}x_{V_k}
+ s_{M_k}
+ A^{old}_{M_k,-V_k}x_{-V_k}\\
&=
A^{old}_{M_k,:}x
+ b_{M_k}
- A^{old}_{M_k,:}x
= b_{M_k}.
\end{aligned}
\end{equation}
For all rows outside $C_k\cup M_k$,
\begin{equation}
A^{new}_{r,:}\tilde{x}=A^{old}_{r,:}x\le b_r .
\end{equation}
Hence every constraint and domain restriction is satisfied by
$\tilde{x}$.
\end{proof}

\section{Additional Experimental Protocol}

\subsection{Datasets and Evaluation Regimes}

The experiments use four MILP families: combinatorial auctions (CA),
capacitated facility location (FA), item placement (IP), and workload
appointment (WA). CA, IP, and WA are time-limit regimes under the
1000-second evaluation budget, so downstream Predict-and-Search (PS)
utility is compared by final primal-dual gap. FA instances are solved
by all PS variants, so runtime reduction is the relevant comparison.

For generation quality, all methods are evaluated at modification
ratios $\eta\in\{0.01,0.05,0.10\}$. The main paper reports the strongest
setting, $\eta=0.10$, while Tables~\ref{tab:app_similarity} and
\ref{tab:app_solve} report the per-ratio results.

Table~\ref{tab:app_dataset_scale} reports the scale of the original
instance collections used in the experiments. The families differ
substantially in both size and sparsity: IP is small but comparatively
dense, FA and WA are much larger and very sparse, and CA lies between
these regimes.

\section{Additional Generation Results}

\begin{table*}[!t]
\centering
\small
\setlength{\tabcolsep}{5pt}
\renewcommand{\arraystretch}{0.98}
\caption{Original dataset scale. Rows, variables, nonzeros, and density are averaged over the original instances; parentheses give the observed min--max range when it varies.}
\label{tab:app_dataset_scale}
\begin{tabular}{@{}lrrrr@{}}
\toprule
\textbf{Dataset} & \textbf{\#Instances} & \textbf{Rows} & \textbf{Variables} & \textbf{Nonzeros / Density} \\
\midrule
CA & 100 & 2603.03 (2505--2685) & 1500.00 & 8330.28 (7745--8959) / 0.00213 \\
FA & 50 & 10201.00 & 10100.00 & 40200.00 / 0.000390 \\
IP & 100 & 195.00 & 1083.00 & 7440.00 / 0.03523 \\
WA & 150 & 64302.17 (64163--64428) & 61000.00 & 361465.11 (346186--375920) / 0.0000922 \\
\bottomrule
\end{tabular}

\vspace{6pt}
\small
\setlength{\tabcolsep}{7pt}
\renewcommand{\arraystretch}{0.98}
\caption{Graph-statistical similarity across modification ratios. Higher is better.}
\label{tab:app_similarity}
\begin{tabular}{@{}llcccc@{}}
\toprule
$\boldsymbol{\eta}$ & \textbf{Method} & \textbf{CA} & \textbf{FA} & \textbf{IP} & \textbf{WA} \\
\midrule
0.01 & Random & 0.857 & 0.000 & 0.419 & 0.163 \\
0.01 & G2MILP & 0.854 & 0.358 & 0.337 & 0.854 \\
0.01 & MILP-StuDio & 0.955 & \textbf{0.823} & 0.687 & 0.937 \\
0.01 & \method & \textbf{0.988} & 0.813 & \textbf{0.972} & \textbf{0.999} \\
\midrule
0.05 & Random & 0.593 & 0.000 & 0.248 & 0.040 \\
0.05 & G2MILP & 0.484 & 0.092 & 0.352 & 0.484 \\
0.05 & MILP-StuDio & 0.936 & 0.723 & 0.687 & 0.910 \\
0.05 & \method & \textbf{0.989} & \textbf{0.813} & \textbf{0.940} & \textbf{0.995} \\
\midrule
0.10 & Random & 0.549 & 0.000 & 0.169 & 0.001 \\
0.10 & G2MILP & 0.249 & 0.091 & 0.336 & 0.249 \\
0.10 & MILP-StuDio & 0.884 & 0.707 & 0.687 & 0.878 \\
0.10 & \method & \textbf{0.976} & \textbf{0.806} & \textbf{0.900} & \textbf{0.992} \\
\bottomrule
\end{tabular}

\vspace{6pt}
\footnotesize
\setlength{\tabcolsep}{3.5pt}
\renewcommand{\arraystretch}{0.96}
\caption{Average solve time and feasible ratio across modification ratios. Each cell reports time in seconds followed by feasible ratio in parentheses.}
\label{tab:app_solve}
\begin{tabular}{@{}llcccc@{}}
\toprule
$\boldsymbol{\eta}$ & \textbf{Method} & \textbf{CA} & \textbf{FA} & \textbf{IP} & \textbf{WA} \\
\midrule
0.01 & Original & 1000.00 (100.0\%) & 3.33 (100.0\%) & 1000.00 (100.0\%) & 940.21 (100.0\%) \\
0.01 & Random & 1000.00 (100.0\%) & 0.01 (0.0\%) & 417.48 (98.0\%) & 0.11 (0.0\%) \\
0.01 & G2MILP & 802.30 (100.0\%) & 0.02 (1.7\%) & 456.19 (49.0\%) & 12.01 (10.0\%) \\
0.01 & MILP-StuDio & 1000.00 (100.0\%) & 2.50 (100.0\%) & 340.02 (68.0\%) & 658.62 (98.0\%) \\
0.01 & \method & 1000.00 (100.0\%) & 3.00 (100.0\%) & 851.38 (89.0\%) & 946.48 (100.0\%) \\
\midrule
0.05 & Original & 1000.00 (100.0\%) & 3.20 (100.0\%) & 1000.00 (100.0\%) & 947.43 (100.0\%) \\
0.05 & Random & 1000.00 (100.0\%) & 0.03 (0.0\%) & 0.15 (98.0\%) & 0.15 (0.0\%) \\
0.05 & G2MILP & 0.69 (100.0\%) & 0.01 (1.8\%) & 0.14 (100.0\%) & 0.01 (0.0\%) \\
0.05 & MILP-StuDio & 1000.00 (100.0\%) & 2.07 (100.0\%) & 340.01 (68.0\%) & 270.39 (92.0\%) \\
0.05 & \method & 1000.00 (100.0\%) & 2.78 (100.0\%) & 715.37 (88.0\%) & 871.51 (98.7\%) \\
\midrule
0.10 & Original & 1000.00 (100.0\%) & 4.05 (100.0\%) & 1000.00 (100.0\%) & 940.11 (100.0\%) \\
0.10 & Random & 1000.00 (100.0\%) & 0.04 (0.0\%) & 0.04 (89.0\%) & 0.24 (0.0\%) \\
0.10 & G2MILP & 0.72 (100.0\%) & 0.01 (2.0\%) & 0.03 (100.0\%) & 0.02 (0.0\%) \\
0.10 & MILP-StuDio & 972.24 (100.0\%) & 1.38 (100.0\%) & 340.02 (68.0\%) & 198.35 (91.0\%) \\
0.10 & \method & 1000.00 (100.0\%) & 1.98 (100.0\%) & 571.83 (81.0\%) & 720.38 (99.3\%) \\
\bottomrule
\end{tabular}
\end{table*}

\section{Additional Downstream Details}

Table~\ref{tab:app_ps_detail} reports solver work units and node counts
for the downstream PS evaluation. These diagnostics are not used as the
primary comparison metric, but they provide additional context about the
search process.

\begin{table*}[!t]
\centering
\small
\setlength{\tabcolsep}{5pt}
\renewcommand{\arraystretch}{0.98}
\caption{Detailed downstream PS diagnostics. Work and node counts are averaged over 40 held-out original test instances per family.}
\label{tab:app_ps_detail}
\begin{tabular}{@{}llrrrr@{}}
\toprule
\textbf{Dataset} & \textbf{Method} & \textbf{Optimal} & \textbf{Time limit} & \textbf{Work units} & \textbf{Nodes} \\
\midrule
CA & Gurobi baseline & 0/40 & 40/40 & 1325.31 & 5371.62 \\
CA & PS baseline & 0/40 & 40/40 & 1268.95 & 5079.77 \\
CA & PS+MILP-StuDio & 0/40 & 40/40 & 1158.28 & 4708.75 \\
CA & PS+\method & 0/40 & 40/40 & 1425.81 & 5751.00 \\
\midrule
FA & Gurobi baseline & 40/40 & -- & 5.075 & 296.25 \\
FA & PS baseline & 40/40 & -- & 4.743 & 296.93 \\
FA & PS+MILP-StuDio & 40/40 & -- & 5.202 & 337.55 \\
FA & PS+\method & 40/40 & -- & 4.767 & 298.03 \\
\midrule
IP & Gurobi baseline & 0/40 & 40/40 & 625.074 & 458155.88 \\
IP & PS baseline & 0/40 & 40/40 & 676.782 & 250009.95 \\
IP & PS+MILP-StuDio & 0/40 & 40/40 & 923.124 & 337235.65 \\
IP & PS+\method & 0/40 & 40/40 & 865.219 & 323478.58 \\
\midrule
WA & Gurobi baseline & 0/40 & 40/40 & 1489.40 & 1176.22 \\
WA & PS baseline & 0/40 & 40/40 & 1364.39 & 1044.20 \\
WA & PS+MILP-StuDio & 0/40 & 40/40 & 1455.24 & 1279.97 \\
WA & PS+\method & 0/40 & 40/40 & 1516.53 & 1284.53 \\
\bottomrule
\end{tabular}

\vspace{6pt}
\small
\setlength{\tabcolsep}{5pt}
\renewcommand{\arraystretch}{0.98}
\caption{Block-pool scale and nominal block size at $\eta=0.10$. Residual nodes per BU approximate the average number of non-interface rows and variables represented by each extracted block unit. The last column is the largest-to-average block-size ratio, not a fraction of total residual size.}
\label{tab:app_block_scale}
\begin{tabular}{@{}lrrrrrr@{}}
\toprule
\textbf{Dataset} & \textbf{\#BU} & \textbf{\#Master} & \textbf{\#Boundary} & \textbf{Residual Nodes / BU} & \textbf{Compat.} & \textbf{Largest / Avg} \\
\midrule
CA & 29.49 & 1275 & 734.2 & 70.98 & 0.719 & 1.739 \\
FA & 974.8 & 307 & 303 & 20.20 & 1.000 & 0.170 \\
IP & 16.00 & 60 & 380 & 52.38 & 1.000 & 4.304 \\
WA & 5754 & 1938 & 1838 & 21.12 & 0.997 & 0.166 \\
\bottomrule
\end{tabular}
\end{table*}

\section{Additional Diagnostics}

The block-pool statistics expose intermediate structure produced by
\method: the number of extracted block units, master constraints,
boundary variables, compatible replacements, and largest-to-average
block-size ratios.
These quantities make the generator inspectable. For example, WA has many small reusable units, IP has a
compact highly compatible pool, and CA has lower compatibility as the
modification ratio increases. These diagnostics help explain why
similarity, feasibility, and compatibility should be read together
rather than as interchangeable indicators.

\subsection{Block-Pool Scale and Nominal Block Size}

Table~\ref{tab:app_block_scale} expands the $\eta=0.10$ block-pool view
with a nominal residual size per block unit. This quantity is computed
from the average original instance scale after removing detected
interface nodes, divided by the average number of extracted block
units. It is a compact scale diagnostic rather than an exact histogram:
the largest-to-average block-size ratio in the last column records how
concentrated the block pool can become.

\endgroup

\end{document}